\documentclass[11pt]{article}
\usepackage[T1]{fontenc}
\usepackage[utf8]{inputenc}
\usepackage[a4paper]{geometry}
\usepackage{amsmath,amssymb,amsthm,graphicx,booktabs,array}
\usepackage{cite}
\usepackage[hidelinks]{hyperref}
\newtheorem{proposition}{Proposition}[section]
\newtheorem{corollary}[proposition]{Corollary}
\title{Learning 3D biophysical cell properties from 2D images and cell-population statistics}
\author{Santiago Hern\'andez-Orozco$^{1}$, Hector Zenil$^{1,2}$\\[2mm]
\small $^{1}$ Algocyte | Oxford Immune Algorithmics, Oxford University Innovation \\ \& 
\small London Institute for Healthcare Engineering, UK\\
\small $^{2}$ Algorithmic Dynamics Lab, Department of Biomedical Computing,\\
\small School of Biomedical Engineering and Imaging Sciences, King's Institute for AI,\\
\small King's College London, UK\\
}
\date{}

\date{}
\begin{document}
\maketitle
\begin{abstract}
Inferring 3D cellular properties from 2D microscopy is difficult when a reference instrument reports only population statistics rather than labels for individual cells. Here we develop a population-supervised framework that maps single 2D red-cell images to latent biophysical quantities and aggregates them to mean corpuscular volume, red-cell distribution width and mean corpuscular haemoglobin. The model combines shared local inference, a biophysically structured decoder for volume and haemoglobin, learned instance weighting and device-specific calibration. We formalise conditions under which aggregate observations identify restricted instance predictors, show why population agreement does not by itself identify single-cell properties or 3D geometry, and derive the dispersion penalty induced by subset mean matching. The development dataset comprises 390 specimens and 1,105 acquisitions across six devices, with reported Pearson correlations of 0.86--0.98 against a Sysmex analyser. The framework provides a testable route from 2D images and population supervision to 3D cellular biophysics without claiming explicit 3D reconstruction.\\

\end{abstract}

A microscope records the 2D appearance of individual cells, whereas a laboratory analyser can provide statistical summaries without measurements paired to those same images. The desired predictor therefore faces two coupled inverse problems: inferring a 3D biophysical quantity from a 2D observation, and learning an individual mapping from population-level supervision. Red-cell volume and haemoglobin content provide a concrete setting: specialised physical methods can measure these quantities at the single-cell level\cite{mohandas1986,roussel2018}, and image cytometry supports automated blood-cell analysis\cite{moradi2023}. Here, ``3D'' denotes volumetric biophysical inference rather than explicit 3D image or shape reconstruction; the training problem arises because paired single-cell labels are unavailable for the microscopy system used here.

Our organising principle is \emph{local inference under population-level supervision}. A shared function estimates properties of one object or one local image region; a specified statistical operation connects those estimates to the available measurement. Restricting where prediction takes place is distinct from restricting where supervision is supplied. A sample-level target can influence every local prediction through backpropagation without making the predictor an unrestricted function of the entire sample. This design yields reusable instance outputs and an explicit measurement pathway. Whether it improves generalisation is an empirical question, not a consequence of locality alone.

The resulting architecture connects to approaches that embed equations, invariances and analytic constraints into neural networks\cite{lagaris1998,raissi2019,ling2016,beucler2021}. We study mean corpuscular volume (MCV), red-cell distribution width as a coefficient of variation (RDW-CV), and mean corpuscular haemoglobin (MCH). The local-inference principle supplies the design rationale; the analysis separates conditions that constrain individual predictions from assumptions required to interpret them physically.

\section{Results}
\subsection{Shared local predictions connect population labels to cellular properties}
Let $\mathcal P=\{\chi_j\}_{j=1}^{M}$ be a collection of populations, with each observed bag represented as a multiset $\chi_j=\{x_{ji}\}_{i=1}^{N_j}$. Multiplicity matters because statistics count repeated observations. For unobserved per-element quantities $y_{ji}$, write the ideal aggregate label as $\xi_j=S(\{y_{ji}\}_i)$, where $S$ is permutation-invariant and may be vector-valued. The general training objective is
\begin{equation}\label{eq:aggregate-loss}
\mathcal L(\theta)=\frac{1}{M}\sum_{j=1}^{M}
\ell\!\left(\xi_j,\,S\bigl(\{f_\theta(x_{ji})\}_{i=1}^{N_j}\bigr)\right).
\end{equation}
Here $f_\theta$ is the same map for every instance and population, and $\ell$ is minimised at agreement. In practice, a reference analyser measures a separate aliquot rather than these exact imaged cells. Equation~\eqref{eq:aggregate-loss} then becomes an empirical approximation to a population measurement relation, with sampling and reference error in addition to prediction error.

The shared map couples the optimisation across bags: its parameters cannot be independently adjusted for each population. At inference, a local predictor receives only the object representation and, where applicable, known device metadata, not a bag embedding or specimen identifier. This restriction does not prevent specimen-related cues from occurring inside a crop. It also does not make thousands of crops thousands of independently labelled biological examples. Rather, each bag exposes many evaluations of a common function under a small number of measured constraints. A low-capacity local map is a design hypothesis for using this signal efficiently; it is not a universal guarantee of lower variance than a suitably restricted bag-level model.

For coupled continuous properties, let $v_\theta(x)$ denote volume and $c_\theta(x)$ concentration in compatible units. Define $f_\theta(x)=(v_\theta(x),c_\theta(x)v_\theta(x))$. The corresponding red-cell measurement operator is
\begin{equation}\label{eq:coupled-operator}
S(\{(V_i,H_i)\})=
\left(\mu_V,\,100\frac{\sigma_V}{\mu_V},\,\mu_H\right),
\qquad H_i=C_iV_i,
\end{equation}
where $\mu$ and $\sigma$ denote the mean and standard deviation under the chosen empirical measure. Thus MCV and RDW constrain the same volume outputs, whereas MCH constrains their products with concentration. The concentration and volume branches can share an encoder. Coupling restricts the computation, but a freely adjustable concentration branch can still compensate for incorrect volume estimates; joint supervision is not by itself proof of individual recovery.

\subsection{From 2D cell images to 3D biophysical quantities}
The model takes segmented $100\times100$ RGB cell images. A shared convolutional encoder and four scalar heads produce latent radius $r_i$, height $h_i$, morphological scaling $s_i$ and concentration-like information $\lambda_i$. The structural decoder forms
\begin{equation}\label{eq:decoder}
\widetilde V_i=\pi s_i r_i^2h_i,\qquad
\widetilde H_i=\lambda_i\widetilde V_i.
\end{equation}
A second image network assigns instance weights; device-specific affine maps convert the two raw quantities to the reference scales before weighted aggregation. All trainable components receive supervision through sample-level indices. This is the weighted, calibrated implementation of the general local-to-aggregate pathway, rather than an additional source of individual labels.

The acquired objects vary in focus, apparent morphology and segmentation quality (Fig.~\ref{fig:cells}). Cellpose supplies segmentation\cite{stringer2021,stringer2025}. Learned weighting, termed \emph{self-referential attention}, controls each object's contribution to the final statistics. It is instance-conditioned weighting, not pairwise transformer self-attention; a high score is not assumed to be calibrated confidence or independently verified optical quality.

The cylindrical factorisation is a tractable geometric surrogate for an intrinsically 3D quantity: cell volume. A comparison of erythrocyte models supports a cylinder as an approximation to volume when dimensions are measured, while distinguishing this from approximation of surface area\cite{udroiu2024}. Here the dimensions are latent. The architecture is therefore \emph{biophysically structured}, without asserting that real cells are cylinders, that their true thickness has been recovered, or that a 3D surface has been reconstructed. The main contribution of this organisation is to expose the intermediate quantities and assumptions to inspection.

\begin{figure}[htbp]
\centering
\includegraphics[width=0.76\linewidth]{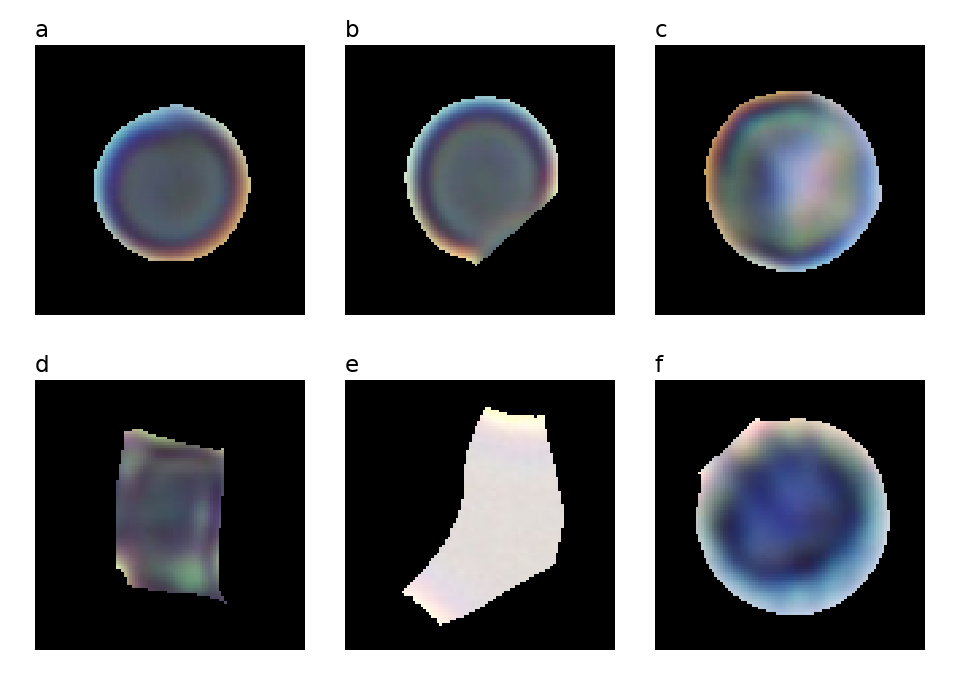}
\caption{Representative $100\times100$ masked RGB input crops illustrating variation in apparent morphology, focus and segmentation. Images are reproduced without intensity adjustment and enlarged by nearest-neighbour replication.}\label{fig:cells}
\end{figure}

\subsection{What population constraints identify}
Correct instance predictions reproduce the ideal target statistics, but the converse requires additional conditions. A positive dispersion target excludes an identical volume for every cell. Nevertheless, fixing the mean and variance of $N\geq3$ unconstrained predictions leaves an $(N-2)$-dimensional sphere of feasible assignments before positivity restrictions. Permuting assignments preserves the entire empirical distribution. Appendix~\ref{app:ambiguity} gives the proof and a synthetic example with identical MCV, RDW and MCH despite a cell-paired volume discrepancy.

The fixed-size overlapping-population argument can be made exact. Let $N$ denote a finite universe of objects, let $A$ be the matrix of normalised bag membership, and let $g$ be a continuous strictly monotone function. For a quasi-arithmetic mean,
\begin{equation}\label{eq:quasi-main}
S_{g,j}(y)=g^{-1}\!\left(\sum_{i=1}^{N}A_{ji}g(y_i)\right).
\end{equation}
Matching genuinely observed bag aggregates is equivalent to $A\widehat z=Az$, with $z_i=g(y_i)$. Full column rank of $A$ identifies all $z_i$, hence all $y_i$. As an idealised example, the collection of all size-$n$ subsets has full column rank for $1\leq n<N$. A stability bound and proof appear in Appendix~\ref{app:quasi}. Crucially, every subset needs its own valid aggregate: copying a parent population's mean onto subsets does not provide these observations.

For disjoint biological populations, a shared restricted function class offers a different route. If $P_j$ is the input distribution for population $j$, mean supervision identifies $f$ within a class $\mathcal F$ precisely when equal values of $\int f\,dP_j$ separate its admissible functions on the observed supports. For fixed linear features $f_\beta(x)=\phi(x)^\top\beta$, full column rank of the matrix with rows $\mathbb E_{P_j}[\phi(X)]^\top$ suffices; its smallest singular value governs stability (Appendix~\ref{app:separation}). These results formalise the value of population diversity without equating it with more repeated crops. These rank conditions are sufficient conditions for identification and do not follow from weight sharing, augmentation or sample count alone.

\subsection{Subsampling is an optimisation design, not new supervision}
Stochastic subsampling exposes different cells and reduces computation, but a random subset need not have its parent's exact mean or dispersion. For fixed predictions $v_1,\ldots,v_N$, a uniformly sampled subset of $m$ distinct objects and a fixed parent target $\mu_0$,
\begin{equation}\label{eq:subset-main}
\mathbb E_S[(\bar v_S-\mu_0)^2]
=(\bar v-\mu_0)^2+\frac{N-m}{m(N-1)}\sigma_v^2,
\end{equation}
where $\sigma_v^2=N^{-1}\sum_i(v_i-\bar v)^2$ is the variance of the full set of predictions. Thus matching subset means adds a dispersion penalty. Non-zero RDW supervision can oppose collapse, but the subset size and relative losses determine the trade-off.

Appendix~\ref{app:subsampling} derives this identity and illustrates its effect by exact enumeration. The result assumes fixed pointwise predictions; training-mode batch normalisation and learned normalised weights introduce additional dependence. Stochastic image augmentation likewise preserves the population target only when the augmented image represents the same physical quantity.

\subsection{Development data and aggregate performance}
The development dataset contains 1,105 imaging acquisitions from 390 unique collected blood specimens across six optical devices. Repeated acquisitions of the same specimen on different devices used separate preparations. Each acquisition yielded 140 RGB fields at $5{,}440\times3{,}648$ pixels; a central region was retained before segmentation into $100\times100$ masked crops. Samples were wet preparations rather than dried smears.

The filed specification reports Pearson correlations against a Sysmex analyser of 0.94--0.96 for MCV, 0.86--0.92 for RDW and 0.94--0.98 for MCH\cite{patent2026}. Table~\ref{tab:correlations} summarises these reported development ranges.

\begin{table}[htbp]
\centering
\caption{Reported Pearson correlations against a Sysmex analyser in the development data\cite{patent2026}.}\label{tab:correlations}
\begin{tabular}{lcc}
\toprule
Index & Reported Pearson correlation & Reference\\
\midrule
MCV & 0.94--0.96 & Sysmex analyser\\
RDW & 0.86--0.92 & Sysmex analyser\\
MCH & 0.94--0.98 & Sysmex analyser\\
\bottomrule
\end{tabular}
\end{table}

Correlation does not establish measurement agreement\cite{bland1986}. Method comparison is more fully characterised by paired errors, bias, limits of agreement and uncertainty while preserving specimen-level dependence. Shared specimen identities across devices also require grouped analysis, and individual-cell accuracy remains distinct from aggregate agreement.

\subsection{Physical interpretation and learned weighting remain distinct}
Even exact volumes do not identify radius, height and scaling separately: $r_i\mapsto a_i r_i$ and $h_i\mapsto h_i/a_i^2$ preserve equation~\eqref{eq:decoder} for every positive $a_i$. Unrestricted positive volume and mass functions can be written with $r=h=1$, $s=V/\pi$ and $\lambda=H/V$. A finite architecture may benefit from factorisation, but the product alone is not an identifying constraint. Appendix~\ref{app:geometry} discusses this freedom and an area-anchored alternative.

Weighted aggregation changes the represented population. For a nonnegative weight $W$ with positive expectation,
\begin{equation}\label{eq:weight-shift}
\frac{\mathbb E[WV]}{\mathbb E[W]}-\mathbb E[V]
=\frac{\operatorname{Cov}(W,V)}{\mathbb E[W]}.
\end{equation}
Weights correlated with genuine cell size can shift MCV and alter dispersion. Quality annotations and morphology-stratified analyses can distinguish attenuation of imaging artefacts from systematic downweighting of genuine abnormal cells (Appendix~\ref{app:weighting}).

\section{Discussion}
The central design choice is to keep prediction local while making supervision a property of the population. The same map is evaluated repeatedly across samples, and its outputs are linked by an explicit measurement operator. This preserves instance-level quantities for inspection and reuse rather than hiding them inside an unconstrained bag regressor. In the red-cell implementation, the model uses 2D appearance to infer volumetric and haemoglobin-related quantities whose joint mean and dispersion are tied to measurable specimen indices. The 3D claim is therefore about volumetric biophysics, not unique geometric reconstruction.

The theory motivates evaluation across distinct biological distributions rather than repeated measurements of the same specimen alone. Comparisons with a direct scalar instance regressor and a matched bag-level baseline can isolate the contribution of structural factorisation, while ablations can separate the effects of joint-statistic supervision, weighting, calibration and subset size.

Population-level agreement, individual-property recovery and geometric reconstruction require progressively different evidence. A matched-cell reference subset, for example using fluorescence exclusion with paired brightfield imaging\cite{roussel2018}, provides a route to assessing instance estimates while retaining image-only deployment. Unpaired histogram agreement cannot establish which volume belongs to which image, and recovery of geometric dimensions requires additional anchors beyond a volume reference.

External-laboratory and device transfer are naturally distinguished by held-out biological specimens and explicit calibration boundaries; labelled calibration on a new device constitutes adaptation rather than zero-shot transfer. The behaviour of learned weighting can likewise be characterised separately for imaging artefacts and genuine morphological extremes. Together, the empirical results and mathematical analysis define a population-supervised route from 2D microscopy to 3D cellular biophysical quantities while making explicit the assumptions required for cell-level and geometric interpretation.

\section{Conclusions}

We introduced a population-supervised framework for learning 3D
biophysical properties of individual cells from 2D microscopy when
reference measurements are available only as statistics of cell
populations. In the red-cell setting, the approach combines shared
single-cell inference with a biophysically structured representation
of volume and haemoglobin, learned instance weighting, and aggregation
to clinically established indices including MCV, RDW-CV and MCH. The
results demonstrate that population-level measurements can provide
sufficient supervision for learning useful cellular biophysical
quantities without requiring paired reference measurements for every
individual cell.

The mathematical analysis clarifies both why this is possible and
where its limits lie. Population statistics impose joint constraints
on a shared instance-level predictor, so sufficiently informative
variation across populations can identify that predictor within a
restricted function class. Conversely, agreement with aggregate
statistics alone does not imply unrestricted recovery of individual
values: mean and dispersion constraints leave non-identifiable
solutions in the absence of additional structure, and repeated random
subsets carrying the same population label do not create independent
information. Likewise, volumetric inference from 2D appearance should
not be confused with unique reconstruction of 3D cellular geometry.
These distinctions separate population-level predictive agreement,
instance-level identifiability and geometric identifiability, and
specify the additional assumptions or measurements required to move
from one level to the next.

The resulting picture is therefore more general than the particular
red-cell application. Measurements need not be paired with the
individual entities whose properties are to be learned. What is
required is a shared mapping, an explicit relation between latent
individual quantities and observed population statistics, and a
sufficiently informative family of populations to distinguish
admissible mappings. Under these conditions, information distributed
across populations can identify properties at the level of their
constituents. The same principle applies recursively whenever
observations and latent variables are organised at multiple levels.

More generally, these results establish a principle of
\emph{hierarchical learning from joint distributions}: supervision
available at one level of organisation can identify representations
and predictions at another when the joint observations provide
sufficient independent constraints. Direct labels for every
individual object, property or task are therefore not necessary.
Multiple jointly observed statistics and their dependencies can
constrain latent structure across levels, allowing information learned
from populations to determine properties of their constituents and,
recursively, information at one hierarchical level to constrain
another. This extends beyond cellular biophysics to a general
framework for learning from indirect and distributed supervision.
Under the identifiability conditions developed here, the same
principle provides a formal basis for zero-shot inference of
unlabelled properties and for learning beyond the specific observables
directly supplied during training.

\section{Methods}
\subsection{Study units, acquisition and preprocessing}
A biological specimen supplies reference analyser values and may contribute multiple independently prepared imaging acquisitions. Each acquisition is a bag of segmented objects; individual crops are not independent biological replicates. The dataset comprises 390 unique specimens, 1,105 acquisitions and six optical devices.

Each wet-preparation acquisition produced 140 RGB images at $5{,}440\times3{,}648$ pixels. The inner 50\% of each field was retained to reduce peripheral defocus. Cellpose version 3.0.8 identified individual objects\cite{stringer2021,stringer2025}, which were represented as $100\times100$ masked RGB crops, yielding on the order of tens of thousands of crops per preparation.

\subsection{Population targets and calibrated measurement operator}
For acquisition $j$, write $\chi_j=\{(x_{ji},d_{ji})\}_{i=1}^{N_j}$, with $x_{ji}\in\mathbb R^{3\times100\times100}$ and known device index $d_{ji}$. An ordinary acquisition has one device; the per-object notation also covers explicitly constructed mixed bags. Its target is $\xi_j=(\mu_{V,j},\mathrm{CV}_{V,j},\mu_{H,j})$: MCV in fL, RDW-CV in percent and MCH in pg. In the ideal population model,
\begin{equation}\label{eq:population-indices}
\mathrm{MCV}=\mathbb E[V],\qquad
\mathrm{RDW\mbox{-}CV}=100\frac{\sqrt{\operatorname{Var}(V)}}{\mathbb E[V]},\qquad
\mathrm{MCH}=\mathbb E[H].
\end{equation}
Throughout this work, the model output for RDW is the coefficient of variation, RDW-CV. RDW-SD is a different quantity and is not interchangeable with RDW-CV.

\subsection{Encoder, scalar heads and pointwise interpretation}
The image network has four convolutional blocks (Table~\ref{tab:architecture}). The first three contain two convolution--batch-normalisation--ReLU sequences, with 32, 64 and 128 channels and max-pooling factors two, two and five. The fourth has 256 channels and adaptive average pooling. A linear layer maps the resulting 256-vector to an embedding of dimension $D$. Four $D\to64\to1$ heads with an intermediate ReLU produce $(r,h,s,\lambda)$. A separate CNN with the same architecture and a scalar head produces the instance weight.

\begin{table}[htbp]
\centering
\small
\caption{Convolutional encoder and scalar-head architecture. $B$ denotes the crop minibatch size and $D$ the embedding dimension.}\label{tab:architecture}
\begin{tabular}{@{}p{0.14\linewidth}p{0.51\linewidth}p{0.27\linewidth}@{}}
\toprule
Stage & Operations & Output shape\\
\midrule
Input & Masked RGB crop & $B\times3\times100\times100$\\
Block 1 & Two $3\times3$ convolutions, BN, ReLU; pool 2 & $B\times32\times50\times50$\\
Block 2 & Two $3\times3$ convolutions, BN, ReLU; pool 2 & $B\times64\times25\times25$\\
Block 3 & Two $3\times3$ convolutions, BN, ReLU; pool 5 & $B\times128\times5\times5$\\
Block 4 & $3\times3$ convolution, BN, ReLU; adaptive pool & $B\times256\times1\times1$\\
Embedding & Flatten; linear $256\to D$; ReLU & $B\times D$\\
Each head & Linear $D\to64$; ReLU; linear $64\to1$ & $B\times1$\\
\bottomrule
\end{tabular}
\end{table}

The weighted-moment formulation assumes positive calibrated volumes and nonnegative instance weights with positive total. The pointwise analysis applies to the deterministic per-image mapping at inference; batch normalisation during training can introduce minibatch dependence and is therefore treated separately in the subset analysis.

\subsection{Structural decoder and device calibration}
After equation~\eqref{eq:decoder}, device-specific affine maps place the raw volume and haemoglobin outputs on the reference scales:
\begin{equation}\label{eq:calibration}
V_{ji}=a_{d_{ji}}\widetilde V_{ji}+b_{d_{ji}},\qquad
H_{ji}=c_{d_{ji}}\widetilde H_{ji}+e_{d_{ji}}.
\end{equation}
These linear transformations convert the raw outputs to the reference units (fL for volume and pg for haemoglobin content) and account for device-specific scale and offset while preserving a shared encoder and decoder. The calibrated outputs are then used to compute the population statistics below.

\subsection{Weighted moments and training objective}
Let $w_{ji}\geq0$ be the learned object weight and $\alpha_{ji}=w_{ji}/\sum_l w_{jl}$. The implemented aggregation is
\begin{align}
\widehat\mu_{V,j}&=\sum_i\alpha_{ji}V_{ji},&
\widehat q_{V,j}&=\sum_i\alpha_{ji}(V_{ji}-\widehat\mu_{V,j})^2,\label{eq:weighted-moments}\\
\widehat\xi_j&=\left(\widehat\mu_{V,j},\,
100\frac{\sqrt{\widehat q_{V,j}}}{\widehat\mu_{V,j}},\,
\sum_i\alpha_{ji}H_{ji}\right).&&\label{eq:predicted-target}
\end{align}
Here $\widehat q_{V,j}$ is the weighted second central moment of the predicted cell volumes. The observation operator of the full model includes learned weighting, unlike the fixed unweighted $S$ used in the elementary identification results.

The three population indices are optimised jointly using mean squared error:
\begin{equation}\label{eq:task-loss}
\mathcal L_{\mathrm{MSE}}=\frac{1}{3M}\sum_{j=1}^{M}\sum_{k=1}^{3}
\left(\widehat\xi_{jk}-\xi_{jk}\right)^2.
\end{equation}

\section{Competing Interests}
The work is associated with Oxford Immune Algorithmics and a filed patent application\cite{patent2026}.

\begingroup
\bibliographystyle{unsrt}
\bibliography{references}
\endgroup

\clearpage
\section*{Appendix}
\appendix
\numberwithin{equation}{section}
\numberwithin{figure}{section}
\numberwithin{table}{section}
\section{What is observed and what is identified?}\label{app:ambiguity}
Let $P_j$ be an image distribution in population $j$. For an ideal deterministic property map $f_0$, suppose its population mean $\mu_j=\mathbb E_{P_j}[f_0(X)]$ is known. A predictor $f$ matches all observed means when $\mathbb E_{P_j}[f(X)]=\mu_j$. In real data, the target is measured with error and the image bag is only a sample from $P_j$; these are additional sources of uncertainty, not assumptions removed by the architecture.

There are three distinct identification problems. First, prediction of an observable population functional, such as mean volume. Second, recovery of the instance map $x\mapsto V(x)$ or $x\mapsto H(x)$. Third, recovery of latent geometric coordinates from those physical properties. The third problem may remain non-identifiable even if the second is solved exactly. Likewise, a model can correctly predict observable population quantities without solving the second problem.

\paragraph{Proposition A.1: the moment fibre.}
For $N\geq3$, fix $\mu\in\mathbb R$ and $q>0$. The set
\[
\mathcal M=\left\{v\in\mathbb R^N:
\frac1N\sum_i v_i=\mu,\quad
\frac1N\sum_i(v_i-\mu)^2=q\right\}
\]
is a sphere of dimension $N-2$ within the mean-constrained hyperplane. Positivity constraints restrict this sphere; at any strictly positive point they do not eliminate its local dimension.

\paragraph{Proof.}
Write $v=\mu\mathbf1+z$. The constraints become $\mathbf1^\top z=0$ and $\|z\|_2^2=Nq$. The first defines an $(N-1)$-dimensional linear subspace and the second a positive-radius sphere in that subspace. At a strictly positive feasible $v$, a sufficiently small neighbourhood preserves positivity. Hence, where such a point exists, local ambiguity persists. $\square$

This statement concerns prediction vectors, not necessarily the representational capacity of a particular finite neural network. It shows that the aggregate observations alone are insufficient. An expressive predictor can realise many assignments on a finite set of distinct training images; a restricted shared predictor might eliminate some of them. Establishing that restriction is part of an identification argument, not a consequence of the number of crops.

Permutations illustrate another boundary. With equal weights, permuting pairs $(V_i,H_i)$ within a bag preserves all moments of their joint empirical distribution, not only MCV, RDW and MCH. Therefore even access to a complete unpaired distribution would not, by itself, identify the image-to-property assignment. With learned weights, jointly permuting $(V_i,H_i,w_i)$ also preserves the aggregate. Independently measured correspondences or informative restrictions across populations are needed.

A non-zero RDW-CV target rules out constant positive volume predictions within that bag, because any such prediction has zero weighted variance whenever the weights have a positive sum. This exclusion is useful but much weaker than recovery of true volumes. MCH adds a mass mean; a free concentration head can still compensate for a volume prediction through $\lambda_i=H_i/V_i$.

For a concrete diagnostic, take $V=(70,80,100,110)$ fL and the reverse assignment $\widehat V=(110,100,80,70)$, with $H=0.34V$ and $\widehat H=0.34\widehat V$ in compatible units. Both give MCV $90$ fL, variance $250\,\mathrm{fL}^2$, RDW-CV approximately $17.5682\%$ and MCH $30.6$ pg. Their cell-paired volume RMSE is $\sqrt{1000}\approx31.6228$ fL. This construction isolates the assignment ambiguity left by aggregate agreement.

\section{Quasi-arithmetic means and overlapping populations}\label{app:quasi}
We formalise identification under quasi-arithmetic mean supervision for overlapping fixed-size populations. The analysis concerns labelled subsets of a fixed universe, not random subsets assigned an unchanged parent label.

Let $I\subset\mathbb R$ be an interval and $g:I\to\mathbb R$ continuous and strictly monotone. A quasi-arithmetic mean of $y_1,\ldots,y_n\in I$ is
\begin{equation}\label{eq:quasi-definition}
S_g(y_1,\ldots,y_n)=g^{-1}\!\left(\frac1n\sum_{i=1}^{n}g(y_i)\right).
\end{equation}
The identity $g(t)=t$ gives the arithmetic mean; $g(t)=\log t$ on positive values gives the geometric mean. This construction is scalar. A vector of moments such as mean and variance is not itself assumed to be a quasi-arithmetic mean.

Consider $N$ fixed objects with labels $y=(y_1,\ldots,y_N)\in I^N$. For each bag $j$, let $m_{ji}$ be the known multiplicity of object $i$, let $n_j=\sum_i m_{ji}>0$, and set $A_{ji}=m_{ji}/n_j$. Then $A\mathbf1=\mathbf1$ and the bag label is $\xi_j=g^{-1}((Az)_j)$ for $z_i=g(y_i)$. Entries of $A$ are known observation weights, not the learned reliability weights of the RBC network.

\begin{proposition}[Identification by transformed aggregate rank]\label{prop:quasi-rank}
If $A$ has full column rank, equality of every quasi-arithmetic aggregate for $\widehat y,y\in I^N$ implies $\widehat y=y$. If $A$ has a non-trivial null space and every $y_i$ lies in the interior of $I$, there are distinct labels arbitrarily close to $y$ with the same aggregates.
\end{proposition}
\begin{proof}
Strict monotonicity makes $g$ injective, so aggregate equality is equivalent to $A(\widehat z-z)=0$. Full column rank forces $\widehat z=z$ and therefore $\widehat y=y$. Conversely, take a non-zero $u\in\ker A$. Because $g$ is continuous and strictly monotone, each $g(y_i)$ is an interior point of the interval $g(I)$. For sufficiently small non-zero $t$, all coordinates of $z+tu$ remain in $g(I)$. Applying $g^{-1}$ coordinatewise gives distinct admissible labels with the same aggregate vector.
\end{proof}

\begin{corollary}[Complete fixed-size subset observations]\label{cor:all-subsets}
For $1\leq n<N$, let $A$ contain the normalised incidence rows of all size-$n$ subsets of $\{1,\ldots,N\}$. Then $A$ has full column rank. Thus genuinely labelled populations of a fixed size greater than one can identify individual labels without singleton labels. The statement for $n>1$ requires $N\geq3$.
\end{corollary}
\begin{proof}
Suppose $Au=0$. For any distinct $i,k$, choose a set $T$ of $n-1$ objects outside $\{i,k\}$, which is possible because $n\leq N-1$. The zero sums on $T\cup\{i\}$ and $T\cup\{k\}$ give $u_i=u_k$. Therefore $u$ is constant. Its sum on a size-$n$ bag is zero, so the constant is zero. When $n=N$, by contrast, only the full-set mean is observed and the matrix has rank one for $N>1$.
\end{proof}

The complete family is sufficient, not necessary. A smaller collection of bags can suffice if its rows have full column rank. For the complete fixed-size family, the diagonal entries of $A^\top A$ are $\binom{N-1}{n-1}/n^2$ and its off-diagonal entries are $\binom{N-2}{n-2}/n^2$, with a binomial coefficient of negative lower index interpreted as zero. Hence
\begin{equation}\label{eq:incidence-spectrum}
\sigma_{\min}(A)^2=\frac{\binom{N-2}{n-1}}{n^2},\qquad 1\leq n<N.
\end{equation}
This expression follows from the eigenvalues of a matrix with constant diagonal and off-diagonal entries. It is for the explicitly unrescaled matrix of all subset rows; duplicating or rescaling observations changes singular values and the corresponding error scale together.

\begin{proposition}[Stability of quasi-arithmetic recovery]\label{prop:quasi-stability}
Assume $A$ has full column rank, the transformed observations are $\widetilde b=Az+e$, and $g^{-1}$ is $K_g$-Lipschitz on an interval containing all coordinates of $z$ and $\widehat z$. Then
\begin{equation}\label{eq:quasi-stability}
\|\widehat y-y\|_2\leq
\frac{K_g}{\sigma_{\min}(A)}
\left(\|A\widehat z-\widetilde b\|_2+\|e\|_2\right).
\end{equation}
A sufficient condition is differentiability with $|g'|\geq c>0$ on the relevant interval, giving $K_g\leq1/c$.
\end{proposition}
\begin{proof}
Use $A(\widehat z-z)=A\widehat z-\widetilde b+e$, the singular-value inequality and the triangle inequality. Applying the coordinatewise Lipschitz bound for $g^{-1}$ yields the stated result.
\end{proof}

\paragraph{Interpretation of the rank condition.}
The result assumes that each aggregate label corresponds to the actual membership of its bag. Reassigning a single population mean to all sampled subsets imposes a different system: for all fixed-size subsets it forces constant transformed predictions rather than the original heterogeneous labels. Disjoint donor bags also lack the object overlap used by this theorem; in that setting, identification depends on restrictions on a shared input-to-property map, as in Appendix~\ref{app:separation}. Learning the observation weights jointly likewise changes the inverse problem. The theorem therefore applies only when the observed bags and aggregate labels satisfy these membership assumptions; the learned weighted model defines a different observation operator.

\section{A population-separation condition and stability bound}\label{app:separation}
\paragraph{Proposition C.1: mean identification in a restricted class.}
Let $\mathcal F$ be a predictor class of functions integrable under every $P_j$. Define the observation operator $T f=(\mathbb E_{P_1}f,\ldots,\mathbb E_{P_J}f)$. If the only functions in $\mathcal F-\mathcal F$ annihilated by $T$ are zero almost surely on the observed population supports, then $Tf=Tf_0$ with $f,f_0\in\mathcal F$ implies $f=f_0$ on those supports.

\paragraph{Proof.}
Set $h=f-f_0$. Equality of observations gives $Th=0$. Apply the stated null-space condition. Conversely, any non-zero $h=f-g\in\mathcal F-\mathcal F$ with $Th=0$ produces indistinguishable predictors $f$ and $g$. $\square$

Thus the separation condition is necessary and sufficient for identification within the specified class, under exact population-mean observations. It does not guarantee generalisation outside the observed support, finite-sample consistency, correct specification or a successful neural-network optimiser.

\paragraph{Corollary C.2: fixed linear features.}
Let $f_\beta(x)=\phi(x)^\top\beta$, where $\phi(x)\in\mathbb R^p$ is fixed, and let $A_{j,:}=\mathbb E_{P_j}[\phi(X)]^\top$. If $A$ has full column rank, exact means identify $\beta$. With noisy observations $y=A\beta_0+\epsilon$, every estimate $\widehat\beta$ satisfies
\[
\|\widehat\beta-\beta_0\|_2
\leq\frac{\|A\widehat\beta-y\|_2+\|\epsilon\|_2}
{\sigma_{\min}(A)}.
\]

\paragraph{Proof.}
Full column rank gives $\|Au\|_2\geq\sigma_{\min}(A)\|u\|_2$. Set $u=\widehat\beta-\beta_0$ and use $Au=A\widehat\beta-y+\epsilon$ and the triangle inequality. $\square$

If an estimated design matrix $\widehat A$ is used instead, the numerator becomes
\[
\|\widehat A\widehat\beta-y\|_2+\|\epsilon\|_2
+\|A-\widehat A\|_{\mathrm{op}}\|\widehat\beta\|_2,
\]
using the true $\sigma_{\min}(A)$ in the denominator. This separates aggregate-target error from uncertainty in sampled feature means; instance accuracy additionally depends on the approximation and sampling terms entering the bound.

For illustration, $A=\left(\begin{smallmatrix}1&-1\\1&1\end{smallmatrix}\right)$ identifies $\beta=(90,10)^\top$ from means $(80,100)^\top$. Replacing both rows by $(1,0)$ destroys slope identification regardless of the number of images used to estimate those rows. This synthetic example illustrates the role of rank independently of the image encoder.

\section{Subsampling and dispersion}\label{app:subsampling}
\paragraph{Proposition D.1: finite-population subset mean loss.}
Let $N>1$ and let $S$ be a simple random sample of $m$ distinct indices from $\{1,\ldots,N\}$. For fixed values $v_i$, define
\[
\bar v=\frac1N\sum_i v_i,
\qquad
\sigma_v^2=\frac1N\sum_i(v_i-\bar v)^2.
\]
For any fixed target $\mu_0$,
\begin{equation}\label{eq:subset-appendix}
\mathbb E_S[(\bar v_S-\mu_0)^2]
=(\bar v-\mu_0)^2+
\frac{N-m}{m(N-1)}\sigma_v^2.
\end{equation}

\paragraph{Proof.}
Let $I_i$ indicate inclusion. Then $\mathbb E I_i=m/N$ and $\mathbb E(I_iI_k)=m(m-1)/(N(N-1))$ for distinct $i,k$. Hence $\mathbb E\bar v_S=\bar v$, and direct expansion gives
\[
\operatorname{Var}(\bar v_S)=\frac{N-m}{m(N-1)}\sigma_v^2.
\]
The result follows from the decomposition of mean squared error into squared bias and variance. $\square$

For independent draws from an image distribution, the corresponding identity is
\[
\mathbb E[(\bar f_m-\mu_0)^2]
=(\mathbb Ef-\mu_0)^2+\frac{\operatorname{Var}(f)}{m}.
\]
Both identities concern fixed pointwise predictions. Batch-dependent transformations or learned normalised weights introduce additional dependence and therefore change the sampling calculation.

Reference-analyser error is separate from subset sampling: treating a measured analyser value as the target defines error relative to that measurement, whereas repeated reference measurements characterise uncertainty relative to the underlying biological quantity.

\paragraph{Exact illustration of dispersion pressure.}
For the synthetic four-cell population $v=(70,80,100,110)$ fL, let $v_i(a)=90+a(v_i-90)$. The mean remains $90$ fL while the variance is $250a^2$ fL$^2$. Enumerating all six size-two subsets gives
\[
\mathbb E_S[(\bar v_S-90)^2]=\frac{250}{3}a^2.
\]
The subset-mean term is therefore minimised at $a=0$, even though $a=1$ reproduces the generating dispersion. Joint RDW supervision supplies an opposing constraint on dispersion.
\begin{figure}[htbp]
\centering
\includegraphics[width=0.76\linewidth]{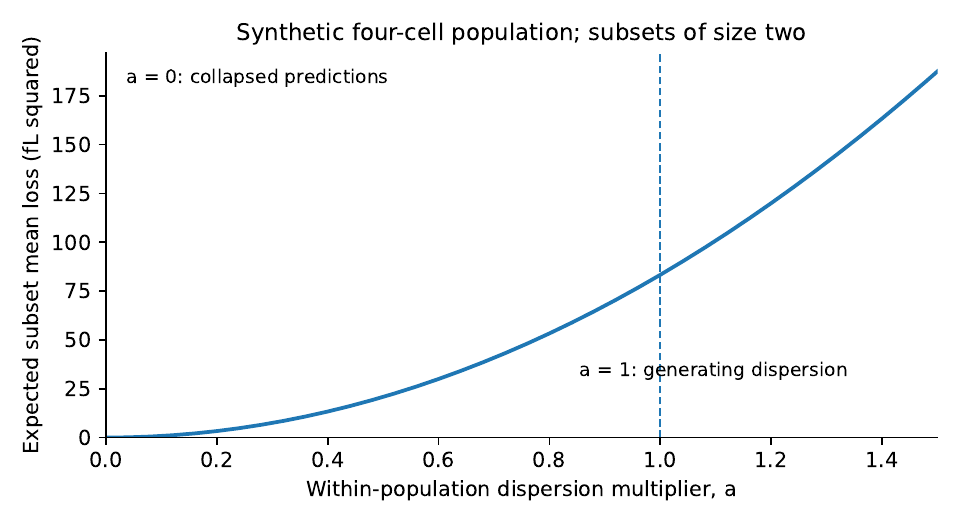}
\caption{Synthetic subsampling diagnostic. All size-two subsets of $(70,80,100,110)$ fL are enumerated after scaling deviations about $90$ fL by $a$. Expected mean-target loss equals $(250/3)a^2$. The vertical reference marks the generating dispersion.}\label{fig:subset}
\end{figure}

\section{Geometric gauge freedom and admissibility}\label{app:geometry}
For positive $(r,h,s,\lambda)$, the raw outputs are $\widetilde V=\pi sr^2h$ and $\widetilde H=\lambda\widetilde V$. For any positive function $a(x)$,
\[
(r,h,s,\lambda)\mapsto(ar,h/a^2,s,\lambda)
\]
leaves both outputs unchanged. Another invariance is $(h,s)\mapsto(bh,s/b)$ for positive $b$. These transformations persist after device calibration because the raw outputs are identical.

For unrestricted scalar functions, every pair of positive target maps $V_0(x),H_0(x)$ has a representation $r=h=1$, $s=V_0/\pi$, $\lambda=H_0/V_0$. Therefore the algebraic product alone does not reduce the class of positive output maps. A finite-capacity architecture or regulariser may make the factorisation useful inductively, but that is an empirical question distinct from identifiability. The argument does not claim that every finite network realises every such function exactly.

One possible extension is to anchor projected area $A_i$ using calibrated image geometry, set an equivalent projected radius $r_i=\sqrt{A_i/\pi}$ and remove or fix the redundant scaling variable. The resulting $V_i=A_i t_i$ predicts an effective thickness. It is not a depth map or a claim that the true cell is a cylinder, and $t_i$ remains unvalidated without a volume reference. Tilt, segmentation bias and optical scale remain relevant sources of error. Bounds on thickness or concentration are physically meaningful only when independently supported for the intended population; narrow normal-range constraints could otherwise suppress pathological extremes.

\section{Learned weights and population shift}\label{app:weighting}
For nonnegative $W$ with $\mathbb EW>0$, finite $\mathbb E|WV|$ and finite covariance,
\[
\frac{\mathbb E[WV]}{\mathbb E[W]}
=\mathbb E[V]+\frac{\operatorname{Cov}(W,V)}{\mathbb E[W]}.
\]
This follows directly from the definition of covariance and is exact for a finite bag when expectations denote uniform averages over its entries. For the illustrative vector $(70,80,100,110)$ and weights $(1,1,1,4)$, the weighted mean is $98.5714$ rather than $90$ fL, a shift of $8.5714$ fL.

The corresponding shift in the raw second moment is $\operatorname{Cov}(W,V^2)/\mathbb E[W]$. Subtracting the square of the weighted mean shows that dispersion depends on both first- and second-moment selection; zero covariance with $V$ alone does not guarantee preservation of variance.

A quality-weighted statistic is appropriate when the target is defined for the selected population. When the intended target is the unweighted biological population, the dependence of weights on morphology or predicted phenotype determines whether weighting changes the estimand. This motivates separate analysis of imaging artefacts and genuine morphological extremes.

\section{Mixture populations and aggregate labels}\label{app:mixtures}
The architecture permits objects from different acquisitions or devices to be aggregated. This does not imply that RDW labels can be linearly interpolated. For a distributional mixture that samples population 1 with probability $\alpha$ and population 2 otherwise,
\[
\mu=\alpha\mu_1+(1-\alpha)\mu_2,
\]
\[
q=\alpha\{q_1+(\mu_1-\mu)^2\}
+(1-\alpha)\{q_2+(\mu_2-\mu)^2\}.
\]
The proof follows by conditioning on the selected component and expanding the squared deviation. The second term includes the between-population variance; omitting it underestimates dispersion. MCH mixes linearly as a mean; RDW-CV is recomputed as $100\sqrt q/\mu$, not mixed linearly.

The identities hold for specified distributions or weighted empirical measures. Drawing finite random numbers of cells from the component specimens adds sampling variability. Constructed labels from noisy analyser moments inherit that uncertainty. Learned weights that alter the realised mixture proportions require further adjustment. For training mixtures, keeping component specimens within the same data partition avoids information leakage. These identities specify the corresponding aggregate labels.

For fixed linear features, exact synthetic mixture means are linear combinations of the original mean-feature rows. If $A$ is the original design and $B$ contains mixture weights, the augmented rows $BA$ have rank at most $\operatorname{rank}(A)$. Adding $BA$ beneath $A$ cannot increase its column rank. Mixtures may aid optimisation or impose other nonlinear consistency constraints, but they do not manufacture new independent linear mean observations. This separates the role of augmentation from that of newly measured biological populations.

\section{Assumptions and scope}\label{app:scope}
The mathematical results address complementary forms of population-supervised inference. The rank theorem concerns exact aggregate labels for overlapping bags with known membership; the population-separation result concerns a restricted shared predictor across distinct population distributions; and the subsampling result quantifies the additional dispersion term introduced when a fixed parent mean is assigned to random subsets. These statements describe different observation operators and are not interchangeable.

The pointwise arguments assume a deterministic per-image map. Batch normalisation in training mode and learned normalised weights can introduce dependence across examples, while positivity and small denominators affect the numerical behaviour of the aggregate statistics. The classification and spatial-field extensions likewise depend on the specified measure: a support-average probability and a cell-count fraction coincide only under additional conditions.

Together, these results delimit when population statistics constrain local predictions and when additional measurements or structural assumptions are required for instance-level or geometric interpretation.

\end{document}